\documentclass[12pt]{article}
 \pdfoutput=1

\usepackage{amssymb,graphicx,amsmath}
\usepackage{amsthm,amsopn,amstext,amsfonts}
\usepackage{color,fullpage,hyperref,xcolor}

\usepackage{graphicx,pgf,tikz,epstopdf}
\usetikzlibrary{fadings,shapes,shapes.misc}

\usepackage{algorithm,algorithmic}
\usepackage{rotating}
\usepackage{caption}
\usepackage{subcaption}

\makeatletter
\newcommand\footnoteref[1]{\protected@xdef\@thefnmark{\ref{#1}}\@footnotemark}
\makeatother

\newtheorem{theorem}{Theorem}
\newtheorem{proposition}{Proposition}
\newtheorem{lemma}{Lemma}

\newtheorem{definition}{Definition}

\usepackage{amsopn}

\newcommand{\TheTitle}{Provable Low Rank Plus Sparse Matrix Separation Via Nonconvex Regularizers}

\title{{\TheTitle}\thanks{This work was supported in part by National Science Foundation under Grant Number DMS-1736326}}
\author{April Sagan\thanks{Department of Biomedical Informatics, School of Medicine, University of Pittsburgh, Pittsburgh, USA;
UPMC Hillman Cancer Center, University of Pittsburgh, Pittsburgh, USA
  (\href{aprilsagan1729@gmail.com}, \url{http://www.aprilsagan.net}).}
\and John E. Mitchell\thanks{Department of Mathematical Sciences, Rensselaer Polytechnic Institute, Troy, NY
  (\href{mitchj@rpi.edu})}}

\ifpdf
\hypersetup{
  pdftitle={Provable Low Rank Plus Sparse Matrix Seperation Via Nonconvex Regularizers},
  pdfauthor={A. Sagan, J. E. Mitchell}
}
\fi

\begin{document}

\maketitle

\begin{abstract}
This paper considers a large class of problems where we seek to recover a low rank matrix and/or sparse vector from some set of measurements.  While methods based on convex relaxations suffer from a (possibly large) estimator bias, and other nonconvex methods require the rank or sparsity to be known a priori, we use nonconvex regularizers to minimize the rank and $l_0$ norm without the estimator bias from the convex relaxation.  We present a novel analysis of the alternating proximal gradient descent algorithm applied to such problems, and bound the error between the iterates and the ground truth sparse and low rank matrices.  The algorithm and error bound can be applied to sparse optimization, matrix completion, and robust principal component analysis as special cases of our results.

\end{abstract}


\textbf{Keywords: }Nonconvex Regularizers, Low Rank Models, Sparse Optimization, Matrix Completion, Robust PCA

\textbf{AMS subject classifications: }62J07, 15A83, 90C26

\section{Introduction}

In order to better understand large datasets and to make inferences about them, it is helpful to understand the underlying patterns in the datasets.  Even when the underlying pattern is highly nonlinear, the data matrix can be approximated as being low rank, an observation that enables techniques to analyze the data in terms of a low dimensional latent space, such as Principal Component Analysis (PCA), identifying outliers through Robust PCA (RPCA), and accurately inferring data points from very few observations of a data matrix through matrix completion.  

Data analysis techniques based upon this low rank property have received much attention in the past decade, with impressive computational results on large matrices and theoretical results guaranteeing the  success of RPCA and matrix completion \cite{candes_tao_2009}\cite{candes_recht_2012}.  Many of these results are based on minimizing the nuclear norm of a matrix (defined as the sum of the singular values) as a surrogate for the rank function, similar to minimizing the $l_1$ norm to promote sparsity in a vector.  

While the convex relaxation is an incredibly useful technique in many applications, minimizing the nuclear norm of a matrix has been shown to introduce a (sometimes very large) estimator bias.  Intuitively, we expect to see this bias because if we hope to recover a rank $r$ matrix, we must impose enough weight on the nuclear norm term so that the $(r+1)$th singular value is zero.  By the nature of the nuclear norm, this requires also putting weight on minimizing the first $r$ singular values, resulting in a bias towards zero proportional to the spectral norm of the noise added to the true data matrix.  

Fortunately, recent work has shown that the estimator bias from convex regularizers can be reduced (or even eliminated, for well conditioned matrices) by using nonconvex regularizers, such as the Schatten-p norm or the minimax concave penalty (MCP).
It has been shown that for sparse optimization, the nonconvexity introduced from these regularizers does not create a further burden in the the optimization process -- in the right circumstances, the nonconvex problem has just one minimizer \cite{Loh_Wainwright_2017}.  Similar results for rank minimization problems have been previously unavailable, a gap that we have aimed to fill in this paper.
\subsection{Summary of Contributions}
In this paper, we focus on the nonconvex, unconstrained optimization problem where we find a low-rank matrix $L\in \mathbb{R}^{d_1\times d_2} $ and sparse vector $s\in \mathbb{R}^{d_s}$.
 \begin{equation}
 \label{eqn:nonconvex_formulation}
 \underset{L,s}{\text{min }}\;
  \frac{\lambda_L}{d_1d_2}\Phi_{\gamma_L}(L)+\frac{\lambda_s}{d_s} \phi_{\gamma_s} (s)+\frac{1}{2n}||\mathcal{A}_L(L)+A_S s-b||_2^2
\end{equation}

The linear mappings $\mathcal{A}_L: \mathbb{R}^{d_1 \times d_2} \rightarrow\mathbb{R}^n$ and $A_s \in\mathbb{R}^{d_s \times n}$ serve as the \textit{observation models} of the underlying low rank matrices and sparse vectors.  Most commonly, we are interested in the observation model 
$$\big[ \mathcal{A}_{\Omega^{obs} }(X)\big]_k=X_{i_k,j_k}$$
for $(i_k, j_k) \in \Omega^{obs}$, where $\Omega^{obs} \subseteq \{1 , \ldots d_1 \} \times \{1, \ldots d_2\}$ is the set of indices where we have a measurement of the low rank matrix we hope to reconstruct.

We denote $\phi: \mathbb{R} \rightarrow \mathbb{R}_+$ to be a concave function used to promote sparsity in both the singular values of $L$ and individual entries in $s$.  We overload the notation to allow for $\phi_{\gamma_s}$ to be a function of a vector $x\in \mathbb{R}^{d_s}$ whose range is $\mathbb{R}_+$, and we denote $\Phi_{\gamma_L}: \mathbb{R}^{d_1\times d_2} \rightarrow \mathbb{R}_+$ as a surrogate to the rank function:  $$\phi_{\gamma_s}(x)=\sum_{i=1}^{d_s} \phi_{\gamma_s}(x_i),\; \: \Phi_{\gamma_L}(X)=\sum_{i=1}^{\text{min}(d_1,d_2)} \phi_{\gamma_L}(\sigma_i(X)) .$$  where $\sigma_i(X)$ denotes the $i$th largest singular value of $X$. We restrict our focus to nonconvex regularizers that are \textit{amenable regularizers}, as described in \cite{Loh_Wainwright_2015}, and defined below.
\begin{definition}A function $\phi_\gamma(t): \mathbb{R}_+ \rightarrow \mathbb{R}_+$ is amenable if it satisfies the following criteria.
\label{as:amenable}
\begin{enumerate}
    \item $\phi_\gamma(0)=0$
    \item $\phi_\gamma$ is non decreasing
    \item For $t>0$, the function $\frac{\phi_\gamma(t)}{t}$ is non increasing in $t$.
    \item the function $\phi_\gamma$ is differentiable for all $t \neq 0$ and subdifferentiable at $ t=0$ with $ \lim_{t \rightarrow 0^+}\phi_\gamma'(t)= 1$.
    \item The function $\phi_\gamma(t)$ is $ \nu$ weakly convex.  That is, the function $\rho_{\nu}:=\phi_\gamma(t)+ \frac{\nu}{2}t^2$ is convex.
\end{enumerate}
\end{definition}


We present a very simple alternating algorithm to find a stationary point of \eqref{eqn:nonconvex_formulation}.  Though similar algorithms have been previously studied and shown to converge, we present a novel analysis of this algorithm and show that not only does this algorithm linearly converge, but it converges to the exact low rank matrix $L^*$ and sparse vector $s^*$ when no Gaussian noise is present.  Furthermore, when Gaussian noise is present, we obtain an error bound that matches the minmax optimal rate.  Our bound greatly improves upon bounds obtained in previous results analysing the convex relaxation and quantify the observation based on computational results that nonconvex regularizers reduce the impact of noise on the quality of the estimator.  
\subsection{Related Works}

The matrix completion problem is as follows: given the values of a low rank matrix for only a sparse set of indices, we seek to determine the rest of the values of the matrix.  While the problem of finding the minimum rank matrix that fits the observations is NP hard in general, it has been shown that under some assumptions, the global minimizer to the convex problem 
\begin{equation*}
         \underset{X\in \mathbb{R}^{d_1 \times d_2}}{\text{min}} ||X||_* \; \text{s.t.} \; X_{ij}=M_{ij} \;  \forall (i,j) \in \Omega^{obs}
\end{equation*}
is exactly $M$, where $\Omega^{obs}$ is the set of indices of $M$ we have observed.  If $M$ is rank $r$ and $\mu-$incoherent (as defined in section 3.1), then with high probability for a set, $\Omega^{obs}$ of $n$ indices chosen uniformly at random, $M$ is the unique minimizer to the convex relaxation so long as $n> C \mu r d_1 \log^{1.2}(d_1)$ for some universal constant $C$ \cite{candes_tao_2009, candes_recht_2012}.  This condition was later improved to $n> C \mu r d_1 \log(d_1)$ \cite{recht_2011}.

Using a convex relaxation for Robust PCA has similar results. While PCA is a powerful technique, it has been shown to be less reliable when just a sparse set of data points are grossly corrupted, and so the goal of RPCA is to identify and remove such corruptions by separating the data matrix into the sum of a low rank and sparse matrix.
\begin{equation*}
        \underset{L,s}{\text{min}} {||L||_* + \lambda_0 ||S||_1 } \:\text{ st } \: L_{ij}+S_{ij}=M_{ij} \; \forall (i,j) \in \Omega^{obs}
    \end{equation*}
The convex relaxation was shown to give the exact solution when every entry of $M$ is observed in \cite{CHANDRASEKARAN_sanghavi_2011}, and when only partially observed (under the same assumptions necessary in matrix completion) by \cite{candes_li_2011}. In contrast to this paper, these works assume that there is measurement noise (besides the sparse corrupted entries).

In many cases of practical interest, our measurements may have some level of noise in addition to being only partially observed or having some corrupted entries.  For matrix completion, we can relax the constraint using a penalty formulation as follows:
$$ \underset{X\in \mathbb{R}^{d_1 \times d_2}}{\text{min}} \lambda ||X||_*  + \sum_{(i,j) \in \Omega^{obs}} \big( X_{ij}-M_{ij}\big)^2 $$
Likewise, RPCA can be formulated as solving:
\begin{align}\label{eqn:rpca_reg}
     \underset{L, S\in \mathbb{R}^{d_1 \times d_2}}{\text{min}}& {\lambda_L||L||_*+ \lambda_s ||S||_1 }  +\sum_{(i,j) \in \Omega^{obs}} \big(L_{ij}+S_{ij}-M_{ij}\big)^2
\end{align}       
Statistical guarantees on the performance of the first of these estimators are discussed in \cite{candes_plan_2010, negahban_wainwright_2012, negahban_wainwright_2011}, and the latter in \cite{agarwal_negahban_2012}.  Specific bounds and assumptions for these works are discussed in Section 4.

In order to reduce the estimator bias for $l_1$ minimization, \cite{candes2008enhancing} proposed an iteratively reweighted $l_1$ norm method to place more weight on minimizing smaller entries, and less on entries further from zero.  This idea was generalized to minimizing any amenable regularizes to promote sparsity.  Theoretical results on the subject include algorithmic guarantees similar to the ones presented in this paper \cite{wang_liu_2014}, and a proof that the nonconvex problem has no spurious local minimizers \cite{Loh_Wainwright_2015, Loh_Wainwright_2017}.

The same regularizers could be used as a surrogate to the rank function, as originally proposed by \cite{mohan_fazel_2012}.  In \cite{Lu_Tang_2016, Lu_Zhu_2015}, the authors propose a generalization of the singular value thresholding algorithm proposed by \cite{cai2008singular}, which was later applied to the problem of RPCA in \cite{Chartrand_2012, Kang_Peng_2015}.  For the problem of matrix completion, the algorithms proposed by \cite{yao_kwok_2019} and \cite{sagan2020lowrank} achieve the fastest computational complexity in other state of the art methods.

Other approaches to low rank optimization rely on, instead of minimizing a surrogate to the rank function, constraining the matrix to be a given rank.  This can be done using constrained optimization to optimize over the set of all rank $r$ matrices \cite{vandereycken2013low}, or by using the low-rank factorization of a matrix $X=UV^T$ for $U\in \mathbb{R}^{d_1 \times r}$ and $V \in \mathbb{R}^{d_2 \times r}$ \cite{lmafitMC, lmafitRPCA}. 

 Previous work on theoretical results pertaining to rank and sparsity constrained methods have consisted of algorithmic guarantees ensuring that we can obtain a matrix sufficiently close to the ground truth low rank and/or sparse matrix for both matrix completion \cite{Jain_Netrapalli_2013} and RPCA \cite{netrapalli2014non, zhang_wang_2018, yi_park_2016}.  Additionally, both of these problems have been shown to have no spurious local minimizers, and so the ground truth matrices are the only minimizers under some assumptions \cite{ge_lee_2016} \cite{ge_jin_2017}.

\begin{table}
\caption{Table of common nonconvex regularizers used for $l_0$ and rank minimization, along with the associated proximal operator.}
\label{table:regularizers}
\centering
\begin{tabular}{l|l|l}
                  & $\phi_\gamma(x)$&   $\text{prox}_{\phi_\gamma}^\tau(y)$ \\ \hline
$l_1$             &$ |x|$  &  $\text{sign}(y)(|y|-\frac{1}{\tau})_+$\\ \hline
\begin{tabular}[c]{@{}l@{}}Capped \\ $l_1$ Norm\end{tabular}  &$\text{min}( |x|, \frac{ c}{2})_+ $&$\begin{cases} \text{sign}(y)(|y|-\frac{1}{\tau}) & |y| \leq \frac{ c}{2} \\ y & \text{otherwise         } \end{cases}$\\ \hline
SCAD              &$\begin{cases}
    |x| & |x| \leq 1 \\
\frac{2 \alpha |x|-x^2-1}{2(\alpha-1)} &  1 \leq |x| \leq \alpha\\
\frac{(\alpha+1)}{2} & |x| > \alpha 
\end{cases}$    & $\frac{1}{\alpha-1}$  $ \begin{cases}
       \text{sign}(y)(|y|-\frac{1}{\tau})_+ & |y| \leq 1 \\
   \frac{\text{sign}(y)\tau(a-1)}{\tau (a-1)-1} (|y|-\frac{a}{\tau})      &  \leq |y| \leq \alpha \\
   y      & |y| > \alpha
\end{cases}$ \\ \hline
MCP& $\begin{cases}  |x|-\frac{x^2}{2}& |x| \leq 1\\ \frac{1}{2}& |x| >1 \end{cases} $& $\begin{array}{l}
   \hspace{-0.2cm} \text{sign}(y)\text{min}\bigg( |y|,  
     \hspace{0.5cm}\frac{ \tau}{ \tau -1} (|y|-\frac{1}{\tau})_+\bigg)\hspace{-.5cm}
\end{array}$ \\ 
\end{tabular}
\end{table}

\section{Alternating Proximal Gradient Descent Algorithm}

Many different methods have been shown to be effective when minimizing nonconvex relaxations of the $l_0$ and rank functions, including iterative reweighted methods, and methods based on low rank factorization.  In this paper, we focus on the most commonly used technique: alternating proximal gradient descent.  

{Consider the objective function $F(x)=\phi(x)+g(x)$, where $g(x)$ is convex and differentiable, and $\phi(x)$ is weakly convex. Instead of minimizing $F(x)$ directly, the proximal gradient descent method approximates $g(x)$ by a quadratic, strongly convex function $\bar{g}_k(x)$ centered about the point $x^k$.
$$\bar{g}_k(x)= g(x^k)+\langle\nabla g(x^k), x-x^k\rangle +\frac{1}{2\tau}||x-x^k||^2$$
At each iteration, we now minimize the function $F_k(x)=\phi(x)+g_k(x)$.  For sufficiently small $\tau$, this function is strongly convex, and the proximal gradient descent algorithm is guaranteed to converge.}

The proximal gradient descent method applied to the function $F(x)=\phi(x)+\bar{g(x)}$ iteratively solves the following problem:
\begin{align*}
    x^{k+1}=\underset{x}{\text{argmin  }} \phi(x)+\frac{1}{2\tau}||x-\big(x^{k}-\tau \nabla g(x^k)\big)||^2:= \text{prox}_\phi^\tau\big(x^{k}-\tau \nabla g(x^k)\big)
\end{align*} 
where we defined the \textit{proximal operator} of a function as the minimum of a combination of the function and the distance from a given point.  For many functions $\phi$ that we are interested in, the proximal operator has a closed form solution, some of which are shown in Table \ref{table:regularizers}.  

For each of the sparsity promoting regularizers in Table \ref{table:regularizers}, the proximal operator is also dubbed as a \textit{shrinkage operator} or a \textit{thresholding operator} because when the input is less than $\frac{1}{\tau}$, the output is 0. Otherwise, the input is moved towards zero or, for some nonconvex regularizers, is unchanged.  So, we can view the proximal gradient algorithm as iteratively taking a step in the gradient direction of $g(x)$, and then applying the proximal operator to promote sparsity.

When applied to the optimization problem in Equation \eqref{eqn:nonconvex_formulation}, we have

\begin{subequations}
\begin{align}
    \tilde{L}^{k+1}=&L^{k}-\tau_L\frac{d_1d_2}{n}\mathcal{A}_L^*\big(\mathcal{A}_L(L^{k}) +A_s s^k-b\big)\\
    L^{k+1}=& \text{prox}_{\Phi_{\gamma_L}}^{{\tau_L}{\lambda_L}}(\tilde{L}^{k+1}) \\
        \tilde{s}^{k+1}=& s^{k}-{\tau_S}\frac{d_s}{n}A_s^T\big(\mathcal{A}_L(L^{k+1}) +A_s s^{k}-b\big)\\
    s^{k+1}=& \text{prox}_{\phi_{\gamma_s}}^{{\tau_s}{\lambda_s}}(\tilde{s}^{k+1})
\end{align}
\end{subequations}
To further simplify the problem, the following proposition will allow $L$ subproblem to be solved in each singular value separately. 
\begin{proposition}
\label{prop:singular_value_prox}
Consider the optimization problem 
\begin{equation} \label{eqn:prox_prop2}
    \underset{X}{\text{min }} \sum_i \phi_\gamma(\sigma_i(X))+ \frac{1}{2\tau}||X-Y||_F^2 
\end{equation} 
where $\phi_\gamma$ is a $\nu$ weakly convex function.  If $\nu< \tau$, then equation \eqref{eqn:prox_prop2} is strongly convex and the minimizer $X^*$ has the same singular vectors as $Y$ with singular values given by 
\begin{align*}
    \sigma_i(X^*)=& \underset{x}{\text{argmin }}  \phi_\gamma(x)+ \frac{1}{2\tau}(x-\sigma_i(Y))^2 \\
    :=& \text{prox}_{\phi_\gamma}^\tau(\sigma_i(Y))
\end{align*}
where $\text{prox}_{\phi_\gamma}^\tau(\sigma_i(Y))$ is the proximal operator.  
\end{proposition}
\begin{proof} 
By convexity of $\phi$, 
$$ \phi(\lambda_i(X_2)) \geq \phi(\lambda_i(X_1)) + \phi'(\lambda_i(X_1)) (\lambda_i(X_2)-\lambda_i(X_1)) $$
Summing over all $i=1, \ldots n$, 
$$ \Phi(X_2) \geq \Phi(X_1) + \langle \nabla \phi( \lambda(X_1)), \lambda(X_1)-\lambda(X_2)\rangle$$
By Corollary 1, 
$$ \Phi(X_2) \geq \Phi(X_1) + \langle U_1 \nabla \phi( \lambda(X_1)) U_1^T, X_1-X_2\rangle$$
And, as $\nabla \Phi(X)= U \nabla \phi( \lambda(X)) U^T$, $\Phi$ is convex.

Consider if the global optimizer $X^*$ was not of the form $U \Sigma V^T$, for a diagonal matrix $\Sigma$.  Let $\bar{X}=U \Sigma^* V^T$, where $\Sigma^*_{ii}=\sigma_i(X^*)$.  By the of Hoffman-Wielandt inequality (Corollary A.1), 
$$||X^*-Y||^2_F\leq \sum_i \big(\sigma_i(X^*)-\sigma_i(Y^*)\big)^2 =||X^*-Y||^2.$$
And, because $\Phi(\bar{X})=\Phi(X^*)$, $\bar{X}$ must also be a global minimizer, contradicting the premise that $X^*$ is the sole global minimizer.
\end{proof}

Proposition \ref{prop:singular_value_prox} tells us that $L^k$ has the singular vectors
of $\tilde{L}^k$  and singular values given by 
$$\sigma_i(L^k)=\text{prox}_{\phi_\gamma}^{{\tau_L}{\lambda_L}}(\sigma_i(\tilde{L}^k)).$$
Likewise, the subproblem in $S$ can be solved in each entry of $s$ individually.
\begin{algorithm}\caption{Alternating Proximal Gradient Descent for Low-Rank Plus Sparse Optimization (APGD)} \label{main_algorithm}
 \begin{algorithmic}
 \FOR{$k=1,\ldots$}
 \STATE $g^{k+1}=-\tau_L\frac{d_1d_2}{n} \mathcal{A}_L^*\big(\mathcal{A}_L(U^k\Sigma^kV^k)+A_ss^k-b\big)$
 \STATE $[U^{k+1}, \tilde{\Sigma}^{k+1}, V^{k+1}]=\text{LRSSVD}(U^{k}, {\Sigma}^{k}, V^{k},  g^{k+1})$
 \STATE ${\Sigma}^{k+1}=\text{prox}\big( \tilde{\Sigma}^{k+1} \big)$
 \STATE $\tilde{s}^k =s^k-\tau_s\frac{d_s}{n} A_s^T \big(\mathcal{A}_L(U\Sigma V)+A_s s^k -b\big)$
 \STATE $s^{k+1}=\text{prox}^{\lambda_s}\big(\tilde{s}^{k+1}\big)$
 \ENDFOR
 \end{algorithmic}
\end{algorithm}

  The slowest operation in the alternating proximal gradient method is by far the singular value decomposition.  However, in practice we can reduce the number of operations by calculating the truncated singular value decomposition only using the first $r_0$ singular values, where $r_0$ is an upper bound on the rank, and enforce that the remaining singular values are zero.  Alternatively, we can calculate each singular value in descending order and stop when a singular value falls below $\lambda_L$, as all remaining singular values will be set to zero by the proximal operator.  So, in the case of RPCA where each entry is observed, each iteration has a computational complexity of $\mathcal{O}(d_1d_2r_0)$, which matches other state of the art methods. 
 
In the case of matrix completion, however, only a sparse set of entries of the low rank matrix are observed, which could be used to increase the efficiency by reducing the amount of computation needed to find the singular value decomposition of $L$ at each iterations.  For a low rank matrix with a low rank factorization $U\Sigma V$, we refer to the problem of finding the SVD of the rank $r$ approximation to the matrix $U \Sigma V +g$ for a sparse matrix $g$ as Low Rank plus Sparse SVD (LRSSVD), originally proposed by \cite{Prateek2010}.

The LRSSVD task can be accomplished efficiently using the same methods as if we were to find the SVD of any other matrix, such as the Power Iteration method.  Recall that the computational complexity of the Power Iteration is limited by the amount of operations needed to multiply the matrix by a vector.  Because the computational complexities of calculating both $u(U\Sigma V^T+Y)$ for $u \in \mathbb{R}^{d_1}$ and $(U\Sigma V^T+Y)v$ for $v \in \mathbb{R}^{d_2}$ are $\mathcal{O}((d_1+d_2)r_0+n)$, we can calculate the top $r_0$ singular values and vectors of $X+Y$ with only $\mathcal{O}\big((d_1+d_2)r_0^2+nr_0\big)$ operations.  
  \begin{algorithm}
 \caption{Singular Value Decomposition for a Low Rank Plus Sparse Matrix}
 \label{alg:SVDsubproblem}
 \begin{algorithmic}[1]
 \renewcommand{\algorithmicrequire}{\textbf{Input:}}
 \renewcommand{\algorithmicensure}{\textbf{Output:}}
 \REQUIRE $U, \Sigma, V, Y$
 \ENSURE Singular value decomposition of $U\Sigma V^T +Y$
 \\ \textit{Initialization}: $\tilde{V}=V$
  \FOR {$k = 1,..,$}
  \STATE $\tilde{U}^{k}=(U\Sigma V^T+Y)\tilde{V}(\tilde{V}\tilde{V}^T)^\dagger$
  \STATE $\tilde{V}^{k}=(\tilde{U}\tilde{U}^T)Z^\dagger \tilde{U}^T(U\Sigma V^T+Y)$
  \ENDFOR
  \STATE $[Q^U, R^U]=\text{QR}(\tilde{U}^k)$, $[Q^V, R^V]=\text{QR}(\tilde{V}^k)$
    \STATE $[U^R, \Sigma^R, V^R]=\text{SVD}(R^U \tilde{\Sigma} {R^V}^T)$
    \STATE $\tilde{U}=Q^U U^R$, $\tilde{V}=Q^V V^R$\\
 \textbf{Return:} $(\tilde{U}, \tilde{\Sigma}, \tilde{V})$
 \end{algorithmic} 
 \label{Agorthim:LRSSVD}
 \end{algorithm}
The other operations in Algorithm \ref{main_algorithm} take no more time than the LRSSVD.  The gradient in the $L$ direction, $g^k$, requires calculating $\Sigma_{ii}^k U^k_i(V_i^k)^T$ for each entry in the support of $\mathcal{A}_L$.  In the case of matrix completion, this is $nr$ operations, which matches the computational complexity per iteration for state of the art matrix completion algorithms. 

\section{Analysis of APGD Algorithm}
In this section, we present the main result of the paper: a recursive bound on the difference of the iterates of the alternating proximal gradient algorithm and the ground truth low rank matrix and sparse vector.  We present the bound for the most general case, and give results on specific problems in the following section.  

\subsection{Restricted Isometry and Orthogonality Properties}
In order to bound the error in the output of our algorithm relative to the underlying ground truth low-rank and sparse matrices $L^*$ and $s^*$, we must first make a number of assumptions about $L^*$, $s^*$, and the observation models $\mathcal{A}_L$ and $A_s$.

First, we must assure that the low rank matrix $L^*$ can be separated from a sparse matrix -- that is, $L^*$ is not sparse itself.  Not only is this necessary for low-rank plus sparse decomposition, but for the problem of matrix completion, this assumption is necessary to assure that a sparse set of observations is a good representation of the entire matrix.  For example, consider the matrix consisting of zeros in every entry besides one entry, where the value is 1. We must observe every entry in the matrix to assure that we can reconstruct the matrix exactly, due to the fact that we must observe the nonzero entry and every entry in its row and column.  To exclude such ill-posed problems from our analysis, we will assume that $L^*$ is \textit{incoherent}, defined as:

\begin{definition}[\cite{candes_tao_2009}]
Let $X \in \mathbb{R}^{d_1 \times d_2}$ be a rank $r$ matrix with singular value decomposition $X=U\Sigma V$, for orthonormal matrices $U\in \mathbb{R}^{d_1 \times r}, V\in \mathbb{R}^{d_2 \times r}$, and diagonal matrix $\Sigma \in \mathbb{R}^{r\times r}$. The tangent space of $X$ is defined as 
\begin{equation}
    \mathcal{T}=\bigg\{  UA+BV \bigg| A\in\mathbb{R}^{r \times d_2}, B\in \mathbb{R}^{d_1 \times r} \bigg\}.
\end{equation}  Furthermore, we say the matrix $X$ (or its tangent space $\mathcal{T}$) is $\mu$ incoherent if
\begin{equation}
||U_{:i}||^2_2 \leq \frac{\mu r}{d_1}, \; ||V_{j:}||^2_2 \leq \frac{\mu r}{d_2} \;\;\; \forall i,j
\end{equation}
\end{definition}

We define the projection of a matrix onto the sparse space $\Omega$ as
\begin{equation}
    \mathcal{P}_\Omega\big(X\big)=\begin{cases} X_{ij} & (i,j) \in \Omega \\ 0 & \text{else} \end{cases}
\end{equation}
and onto the tangent space $\mathcal{T}$ as
\begin{equation}
    \mathcal{P}_\mathcal{T}\big(X\big)=UU^T X+XVV^T-UU^TXVV^T
\end{equation}

Next, we discuss the conditions that the observation models $\mathcal{A}_L$ and $A_s$ must satisfy in order to recover the ground truth low rank and sparse matrix, known as the restricted isometry property.  Loosely, the RIP states that for any two vectors in $\Omega$ (or matrices in $\mathcal{T}$), we can obtain a sufficiently accurate estimate of the distance between the two through the observation model $A_s$ (or $\mathcal{A}_L$).

The RIP was originally proposed for sparse vectors by Candès and Tao in \cite{candes_tao_05}.  Original attempts to extend the property to low rank matrices failed to consider coherent matrices, and thus had little practical applications.  Candès and Tao \cite{candes_tao_2009} later introduced the incoherence assumption and proved that it applied to the problem of matrix completion.  Here, we give two versions of a definition of restricted isometry property, one for sparse vectors and one for matrices that have a low-rank, incoherent tangent space.

\begin{definition}
\label{as:rip} The linear mapping $A_s$ satisfies the $(\alpha, \kappa)$ sparse Restricted Isometry Property if, for any $x$ satisfying $||x||_0 \leq \alpha d_s$
    \begin{align*}
        (1-\kappa_S) ||x||^2 &\leq \tau_s \frac{n}{d_s}||A_s x||^2  \leq (1+\kappa_S) ||x||^2  
    \end{align*}
for some constant $\tau_s$.  Likewise, the linear mapping $\mathcal{A}$ satisfies the $(\mu, r, \kappa)$ low rank Restricted Isometry Property if for any $X$ in a $\mu$-incoherent rank $r$ tangent space $\mathcal{T}$,
\begin{align*}
            (1-\kappa_L) ||X||_F^2 &\leq \tau_L\frac{n}{d_1d_2} ||\mathcal{A}_LX||_F^2  \leq (1+\kappa_L) ||X||_F^2 
\end{align*}
for some constant $\tau_L$.
\end{definition}

In some cases, it may be more useful to use the following characterization of the RIP, which bounds the difference between the operator $\tau \mathcal{A}^*\mathcal{A}$ and the identity operator when restricted to sparse vectors or low-rank and incoherent matrices.
\begin{proposition}
For a matrix $A_s\in \mathbb{R}^{n\times d_s}$ satisfying the $(\alpha, \kappa_s)$ sparse RIP,
    \begin{align*}
        || \frac{\tau_sn}{d_s}\mathcal{P}_\Omega A_s^TA_s\mathcal{P}_\Omega - \mathcal{P}_\Omega||^2 \leq \kappa_s.
    \end{align*}
Likewise, for any linear mapping $\mathcal{A}_L:\mathbb{R}^{d_1\times d_2}\rightarrow\mathbb{R}^n$ satisfying the $(\mu, r, \kappa_L)$ low rank RIP,
    \begin{align*}
                ||\tau_L\frac{d_1d_2}{n} \mathcal{P}_\mathcal{T} \mathcal{A}_L^* \mathcal{A}_L\mathcal{P}_\mathcal{T}- \mathcal{P}_\mathcal{T}||^2 \leq \kappa_L
    \end{align*}

\end{proposition}

Finally, we discuss the interplay between the set of sparse matrices and the low rank, incoherent tangent space, and their observation models.  We hope to be able to separate the measurement vector $b$ into two parts: one in the span of $\mathcal{A}_L \mathcal{P}_\mathcal{T}$, and one in the span of $A_s \mathcal{P}_\Omega$.  In order to achieve this quickly, we require that there are no non-trivial vectors in the intersection of the two sets, which is equivalent to saying that $||\mathcal{P}_\Omega A_s^T \mathcal{A}_L \mathcal{P}_\mathcal{T}|| <1$.  Under some assumptions, this norm is actually close to zero, a concept we refer to as \text{restricted orthogonality}, which we define here and verify that it applies to the problems we are interested in in Section 4.


\begin{definition}
Linear maps $\mathcal{A}_L$ and $A_s$ satisfy the  $\kappa-$ restricted orthogonality property over the sets $\mathcal{T}$ and $\Omega$ (respectively) when
    \begin{align*}
      \frac{d_s}{n} || \mathcal{P}_\Omega A_s^* \mathcal{A}_L\mathcal{P}_\mathcal{T}||^2 \leq \kappa,\;\frac{d_1d_2}{n} || \mathcal{P}_\mathcal{T} \mathcal{A}_L^* A_s\mathcal{P}_\Omega||^2 \leq \kappa
    \end{align*}
\end{definition}

\subsection{Main Result}
Define the difference between the iterates of the alternating proximal gradient descent algorithm and the ground truth low rank matrix and sparse vector at iteration $k$ as $\Delta_L^k=L^*-L^k$ and $\Delta_s^k=s^*-s^k$.  Our main result in the most general form gives a bound on the norm of $\Delta_L^k$ and $\Delta_s^k$ in terms of the differences at the previous iteration, $\Delta_L^{k-1}$ and $\Delta_s^{k-1}$. 
\begin{theorem} \label{thm:main} Let $L^k$ and $s^k$ be the sequences generated by Algorithm \ref{main_algorithm}.  Assume that $$b=\mathcal{A}_L(L^*)+A_s s^*+\mathcal{E} \in \mathbb{R}^n,$$ where $L^* \in \mathbb{R}^{d_1 \times d_2}$ is a rank $r$ and $\mu-$ incoherent matrix, and $s^*\in \mathbb{R}^{d_s}$ is a sparse vector with $supp(s^*)= \Omega$, and the linear mappings $\mathcal{A}_L$ and $A_s$ satisfy the $(2r, 3\mu, \kappa_L)-$ low rank RIP and the $(\alpha, \kappa_s)-$ sparse RIP respectively, and together satisfy the ROP with constant $\kappa$.  If $\lambda_L \geq ||\mathcal{A}_L^*(\mathcal{E})||_2+||\mathcal{A}_L^*A_s \Delta_s^{k-1}||_2$, then
\begin{align*}
    ||\Delta_L^{k+1}&||_F^2\leq  \kappa_L||\Delta_L^k||_F^2+\kappa\tau_L||\Delta_s^k||^2 +\frac{\tau_Ld_1d_2}{n}||\mathcal{P}_\mathcal{T} \mathcal{A}_L^* \mathcal{E} ||_F^2 +\lambda_Lr\phi'(\sigma_r(L^*))
\end{align*}
Likewise,  if  $\lambda_S \geq ||A_s^T(\mathcal{E})||_\infty+||A_s^T\mathcal{A}_L \Delta_L^{k-1}||_\infty$ and $s_{min}$ is the smallest non-zero value of $s^*$, then 
\begin{align*}
    ||\Delta_s^{k+1}||^2&\leq  \kappa_s||\Delta_s^k||^2+\kappa\tau_s||\Delta_L^{k+1}||_F^2 +\frac{\tau_sd_s}{n}||\mathcal{P}_\Omega A_s^T \mathcal{E} ||^2 +\lambda_s \alpha d_s \phi'(s_{min}-\lambda_s) 
\end{align*}

\end{theorem}
For each of these bounds, we can think of the third term as the estimation error introduced by the noise, and the fourth term as the approximation error, which accounts for the bias in the regularizer proportional to the derivative of the regularizer.  Previous results for the nuclear norm and $l_1$ norm give similar bounds, but make the concession that the approximation error is the dominating term.  Under some circumstances, that term is equal to zero in our bound.

\subsection{Proof of Main Result}
We start by presenting the the following two lemmas regarding the proximal operator for the low rank regularizers, which we prove in the following section.

\begin{lemma} \label{lem:low_rank_prox}
Let $L^*$ be a rank $r$, $\mu$-incoherent matrix whose singular vectors form the tangent space $T$, and $\bar{L}$ be defined as
\begin{equation}
    \bar{L}=\underset{L\in \mathbb{R}^{d_1 \times d_2}}{\text{argmin}} \lambda \Phi(L)+||L-L^*+\delta||_F^2
\end{equation}
where $\Phi$ is an at most $\frac{1}{\lambda}$ weakly convex regularizer satisfying Assumption \ref{as:amenable}, and $\delta \in \mathbb{R}^{d_1 \times d_2}$ satisfies $||\delta||_2 \leq \lambda$.  Define $\Delta_L=\bar{L}-L^*$. Then, 
\begin{equation}
    ||\Delta_L||^2_F \leq 2||\mathcal{P}_\mathcal{T}\big(\delta\big)||^2_F+ \lambda r \phi'\big(\sigma_r(L^*)\big)
    \end{equation}
\end{lemma}

In order to utilize the RIP and ROP conditions, we need to verify that $L^{k+1}$ is low rank and incoherent, and that $s^{k+1}$ is sparse, which we will do inductively.  Assume that $L^k$ is at most rank $r$, and that its tangent space is $2\mu$ incoherent.  Additionally, assume that $\text{supp}\big(S^k) \subseteq \text{supp}\big(S^*)$.  Clearly, these conditions are met at the first iteration by initializing the algorithm with $L^0=0$ and $S^0=0$.  

At iteration $k+1$, 
\begin{align*}
    L^*&-\bar{L}^{k+1}\\=&L^*-\big(L^k-\tau_L\frac{d_1d_2}{n} \mathcal{A}_L^*\big(\mathcal{A}_L(L^k) +A_ss^k-b)\big)\big) \\
    =& L^* - L^k +\tau_L\frac{d_1d_2}{n} \mathcal{A}_L^*\mathcal{A}_L L^k + \tau_L\frac{d_1d_2}{n}\mathcal{A}_L^*A_s s^k-\tau_L\frac{d_1d_2}{n}\mathcal{A}_L^*\big(\mathcal{A}_L(L^*) +A_s(s^*)+ \mathcal{E}\big) \\
    =& \big(\Delta_L^{k}- \tau_L\frac{d_1d_2}{n}\mathcal{A}_L^*\mathcal{A}_L \Delta_L^k\big) - \tau_L\frac{d_1d_2}{n} \big( \mathcal{A}_L^* A_s \Delta_S^k -\mathcal{A}_L^* \mathcal{E}\big) 
\end{align*}
where the first equality comes from the definition of $b$, and the second inequality substitutes $\Delta_S^k=S^*-S^k$ and $\Delta_L^k=L^*-L^k$.  By Lemma \ref{lem:low_rank_prox} (along with the triangle inequality), we have that
\begin{align*}
    ||\Delta_L^{k+1}||_F^2\leq & ||\mathcal{P}_\mathcal{T}\big(\Delta_L^{k}- \tau_L\frac{d_1d_2}{n}\mathcal{A}_L^*\mathcal{A}_L \Delta_L^k\big)||_F^2+\tau_L\frac{d_1d_2}{n}||\mathcal{P}_\mathcal{T} \mathcal{A}_L^* A_s \Delta_S^k ||_F^2\\& +\tau_L\frac{d_1d_2}{n}||\mathcal{P}_\mathcal{T} \mathcal{A}_L^* \mathcal{E} ||_F^2  +\lambda_Lr\phi'\big(\sigma_r(L^*)-2{\lambda_L}{\tau_L}\big) 
\end{align*}
Let $\mathcal{T}^k$ denote the union of the tangent space of the rank $r$ approximation $L^k$ and $\mathcal{T}$ so that $\Delta_L^k \in \mathcal{T}^k$.  Then,
\begin{align*}
||\mathcal{P}_\mathcal{T}\big(\Delta_L^{k}-\tau_L\frac{d_1d_2}{n} \mathcal{A}_L^*\mathcal{A}_L \Delta_L^k\big)||_F^2&\leq ||\mathcal{P}_{\mathcal{T}^k}\big(\Delta_L^{k}- \tau_L\frac{d_1d_2}{n}\mathcal{A}_L^*\mathcal{A}_L \Delta_L^k \big)||_F^2\\
&\leq ||\mathcal{P}_{T^k}\big(\tau_L\frac{d_1d_2}{n}\mathcal{A}_L^*\mathcal{A}_L -\mathcal{I}\big)\mathcal{P}_{\mathcal{T}^k}\Delta_L^k||_F^2 \leq \kappa_L ||\Delta_L^k||_F^2
\end{align*}
where the first inequality comes from the contractive property of $\mathcal{P}_\mathcal{T}$, and the second inequality comes from the fact that $\mathcal{P}_{\mathcal{T}^k} \Delta_L^k=\Delta_L^k$.  And, because $\mathcal{T}^k$ is has incoherence at most $3 \mu$, we can apply the low-rank RIP to obtain the third inequality.  

By the inductive hypothesis stating that $\text{supp}\big(S^k) \subseteq \text{supp}\big(S^*)$ (and thus, that $S^k$ is $\alpha$-sparse), we can use the ROP to claim that 
$$\frac{d_1d_2}{n}||\mathcal{P}_\mathcal{T} \mathcal{A}_L^* A_s \Delta_S^k ||_F^2 \leq \kappa ||\Delta_S^k ||_F^2.$$
Combining these gives the desired bound on $\Delta^k_L$.  

Now, we must show that $L^{k+1}$ is rank $r$ and $2\mu$ incoherent.  By Weyl's inequality (\cite{horn_johnson}, see Appendix A), we know that 
\begin{align*}
    \sigma_{r+1}(\tilde{L}^{k+1})\leq& \sigma_{r+1}(L^*)+||L^*-\tilde{L}^{k+1}||_2\\
    \leq & ||\Delta_s^{k}||_2+||\mathcal{A}_L^* \mathcal{E}||_2
\end{align*}
Because this is less than $\lambda_L$ by our assumptions, $L^{k+1}$ must be rank $r$. 

In order to show that $\tilde{L}^{k+1}$ is at most $2\mu$ incoherent, we apply the following Lemma, which uses the Davis-Kahan inequality \cite{davis-kahan}.
\begin{lemma}\label{lem:incoherence}
If $X\in \mathbb{R}^{d_1\times d_2}$ (with $d_1 \geq d_2$) is a rank $r$, $\mu-$ incoherent matrix, and $\Delta\in \mathbb{R}^{d_1\times d_2}$ satisfies $||\Delta||_2 \leq \frac{1}{2} \frac{\mu r}{d_1}\sigma_r(X)$, then the top $r$ singular vectors of the matrix $X+\Delta$ form a $2\mu-$ incoherent tangent space. 
\end{lemma}

Next, we  bound $\Delta_s^k$ and showing $s^{k+1}\in \Omega$ in a similar manner.  We present the following Lemma, which mirrors Lemma \ref{lem:low_rank_prox}.
\begin{lemma} \label{lem:sparse_prox}
Let $s^*\in \mathbb{R}^{d_s}$ be an sparse vector with support $\Omega$, and let $\bar{S}$ be defined as 
\begin{equation}
    \bar{s}=\underset{s\in \mathbb{R}^{d_s}}{\text{argmin }}\; \lambda \phi(s)+||s-s^*+\delta||^2
\end{equation}
Where $\phi$ is an at most $\frac{1}{\lambda}$ weakly convex amenable regualarizer, and $\delta \in \mathbb{R}^{d_s}$ satisfies $||\delta||_\infty \leq \lambda$.  Then, $\text{supp}(\bar{s})\subseteq \Omega$. Furthermore, we can bound the difference $\Delta_s=\bar{s}-s^*$ as
\begin{equation}
    ||\Delta_s||^2_F \leq ||\mathcal{P}_\Omega\big(\delta\big)||_F+2 \lambda \sqrt{|\Omega|} \phi'\big(s_{min}-\lambda\big).
    \end{equation}
\end{lemma}
By the update equation for $\tilde{s}^{k+1}$, we have
\begin{align*}
    s^*-\tilde{s}^{k+1}=& s^*-\big(s^k- \mathcal{A}_s^*\big(\mathcal{A}_s(s^k) +\mathcal{A}_L(L^k)-b)\big)\big) \\
    =& s^* - s^k + \frac{\tau_s}{d_s}\mathcal{A}_s^*\mathcal{A}_s s^k + \frac{\tau_s}{d_s}\mathcal{A}_L^*\mathcal{A}_L L^k-\frac{\tau_s}{d_s}A_s^*\big(\mathcal{A}_s(s^*) +\mathcal{A}_L(L^*)+ \mathcal{E}\big) \\
    =& \big(\Delta_s^{k}- \frac{\tau_s}{d_s}\mathcal{A}_s^*\mathcal{A}_s \Delta_s^k\big) - \frac{\tau_s}{d_s} \mathcal{A}_s^* \mathcal{A}_L \Delta_L^{k+1} -\frac{\tau_s}{d_s}\mathcal{A}_s^* \mathcal{E}
\end{align*}
By Lemma \ref{lem:sparse_prox},
\begin{align*}
    ||\Delta_s^{k+1}||_F^2\leq & ||\mathcal{P}_\Omega\big(\Delta_s^{k}- \mathcal{A}_s^*\mathcal{A}_s \Delta_s^k\big)||_F^2+||\mathcal{P}_\Omega \mathcal{A}_s^* \mathcal{A}_L \Delta_L^{k+1} ||_F^2 \\&+||\mathcal{P}_\Omega \mathcal{A}_s^* \mathcal{E} ||_F^2 +\lambda_s|\Omega| \phi_s'(s_{min}-\lambda_s)
\end{align*}
Applying the RIP to the first term gives:
$$||\mathcal{P}_\Omega\big(\Delta_s^{k}-\frac{\tau_s}{d_s} \mathcal{A}_s^*\mathcal{A}_s \Delta_s^k\big)||_F^2\leq \kappa_s||\Delta^{k}_s||_F^2$$
And, applying the ROP to the second term gives us:
$$||\frac{\tau_s}{d_s}\mathcal{P}_\Omega \mathcal{A}_s^* \mathcal{A}_L \Delta_L^{k+1} ||_F^2=||\frac{\tau_s}{d_s}\mathcal{P}_\Omega \mathcal{A}_s^* \mathcal{A}_L\mathcal{P}_{T^{k+1}} \Delta_L^{k+1} ||_F^2 \leq \frac{\tau_s}{d_s}\kappa||\Delta_L^{k+1} ||_F^2 $$
Combining these terms gives us the desired result.

\subsection{Proofs of Supporting Lemmas}

\begin{proof}[Proof of Lemma \ref{lem:sparse_prox}]

 First, we will show that if $s^*_{i} = 0$, then $\bar{s}_{i} = 0$.  For $i$ not in $\Omega$, 
 $$\bar{s}_i =\text{argmin}_s \;\;\lambda \phi(s)+\frac{1}{2}(s-\delta_i)^2$$
The sub-gradient of the objective function evaluated at zero is 
$\{\delta_i+g\;\big|\; |g|\leq \lambda\}$.  By our assumption that $\lambda\geq \delta_i$, 0 must be in this set, and so $\bar{s}_i=0$ is a stationary point.  Because the objective function is strongly convex by assumption, this is a global minimizer.

 Next, consider if $i \in \Omega$.  By first order necessary conditions for optimality, 
\begin{align*}
\bar{s}_i=s_{i}^*+ \delta_i -\phi'(\bar{s}_i) 
\leq s_{i}^*- \delta_i -\phi'(s_{ij}^*+ \delta_i)
\end{align*}
This gives the bound:
$$|\bar{s}_i-s_i^*|\leq |\delta_i|+\phi'(s_{min}-\lambda)$$
Combining these facts, 
\begin{align*}
    ||\bar{s}-s^*||^2=||\mathcal{P}_\Omega(\bar{s}-s^*)||^2
    \leq||\mathcal{P}_\Omega(\delta)||^2+|\Omega| \phi'(s_{min}-\lambda)^2
\end{align*}

\end{proof}

\begin{proof}[Proof of Lemma \ref{lem:low_rank_prox}]
Let 
\begin{align*}
    f(L)=\lambda \Phi(L)+\frac{1}{2}||L-L^*+\delta||_F^2
\end{align*}
be the function that $\bar{L}$ is the global minimizer of, that is, $0 \in \partial f(\bar{L})$.  The subdifferential of $f$ at $L^*$ is as follows: 
$$\big\{\lambda\big(U\Sigma^\phi  V^T+W\big)+\delta\:\bigg|\: W\in T^\perp,||W||_2 \leq 1 \big\} $$
where $L^*=U\Sigma V^T$ and $$\Sigma^\phi=\text{diag}(\phi'(\sigma_1(L^*)),\phi'(\sigma_2(L^*)),\ldots, \phi'(\sigma_r(L^*))).$$  
Because $\Phi$ is $\frac{1}{2\lambda}$-weakly convex, $f$ is $\frac{1}{2}$-strongly convex, which gives us:
\begin{align*}
||L^*-\bar{L}||_F\leq &\inf 2 ||\partial f(L^*)||_F.
\end{align*}
Consider the subgradient given by $W=-\frac{1}{\lambda}\mathcal{P}_{\mathcal{T}^\perp}(\delta)$.  Note the that $||W||_2\leq 1$ as we assume $\lambda \geq ||\delta||_2$.
\begin{align*}
||L^*-\bar{L}||_F\leq &2 ||\lambda U\Sigma^\phi  V^T+\delta-\mathcal{P}_{\mathcal{T}^\perp}(\delta)||_F\\
\leq&2 \lambda|| U\Sigma^\phi  V^T||_F+2||\delta-\mathcal{P}_{\mathcal{T}^\perp}(\delta)||_F\\
\leq&2 \lambda\sqrt{r} \phi'(\sigma_r(L^*))+2||\mathcal{P}_{\mathcal{T}}(\delta)||_F
\end{align*}
The second inequality is the triangle inequality, and the third uses the fact that $||\Sigma^\phi||_2 =\phi'(\sigma_r(L^*))$
\end{proof}

\begin{proof}[Proof of Lemma \ref{lem:incoherence}]
Let ${U}\in \mathbb{R}^{d_1\times r}$ and $\tilde{U}\in \mathbb{R}^{d_1\times r}$ denote the (top $r$) left singular vectors of $X$ and $X+\Delta$ respectively.  By the Davis Kahan theorem, 
$$ \text{dist}\big(U, \tilde{U}\big) \leq \frac{||\Delta||_2}{\sigma_r(X)+ ||\Delta||_2} \leq \frac{1}{2}\frac{\mu r}{d_1}$$
where the second inequality uses the fact that $||\Delta||_2\leq \frac{\mu r}{2d_1} \sigma_r(X)$
Let $u_i$ and $\tilde{u}_i$ be the $i^{th}$ left singular vectors of $X$ and $X+\Delta$ respectively, and let $\theta_i$ be the angle between the vectors.  
$$\max(\tilde{u}_i) \leq \max(u_i)+2 \sin (\theta_i) \leq \frac{\mu r}{d_1}+\frac{\mu r}{d_1} =\frac{2 \mu r}{d_1}$$
where we use the fact that$$\text{dist}\big(U, \tilde{U}\big) =\text{max}_i \: \sin(\theta_i) $$ So, the rank $r$ approximation of $X+\Delta$ is $2\mu$ incoherent.  
\end{proof}
\section{Results for Specific Models}
In this section, we use Theorem \ref{thm:main} to analyze an application of the alternating proximal gradient descent algorithm to the problems of matrix completion and RPCA.
\subsection{Matrix Completion}
We start by considering the problem of matrix completion.  Here, we have a sparse set of observed entries of $L$, $\mathcal{A}_L =\mathcal{A}_\Omega$, and we do not consider a sparse vector $s$ (i.e. $A_s=0$).

We present a version of the RIP for the sampling operator from \cite{recht_2011}.  

\begin{lemma}[\cite{recht_2011}] \label{lem:recht}
Let $\Omega$ be a set of $n$ entries of $  \{1,\ldots, d_1\} \times \{1,\ldots,d_2\}$ drawn independently at random with uniform probability, with $n> 64 \mu r (d_1+d_2) \log(d_2)$.  Then, with probability at least $1-2d_2^{-2}$, 
$$ \frac{5}{6}||X||_F^2 \leq \frac{d_1d_2}{n} ||\mathcal{A}_\Omega (X)||^2 \leq \frac{7}{6} ||X||_F^2$$
for any rank $r$, $\mu-$incoherent matrix $X$.
\end{lemma}

Additionally, we will assume that the additive noise $\mathcal{E}$ has entries that are mean zero i.i.d.\ variables.  The effect the noise has on the estimator is reduced due to the fact that very little of $\mathcal{A}^*_\Omega(\mathcal{E})$ will lie in the tangent space $\mathcal{T}$.  To formalize this intuition, we cite the following lemma from \cite{zhang_wang_2018}.

\begin{lemma}
\label{lem:noise} Assume that the entries of $\Omega$ are chosen uniformly at random from $\{1,\ldots,d_1\}\times\{1,\ldots,d_2\}$, with $d=\text{max}(d_1,d_2)$, and the entries of $\mathcal{E}$ are mean zero i.i.d. random variables with variance $\nu^2$. For some universal constants $C_1$ and $C_2$, the following hold:
$$||\mathcal{A}^*_\Omega(\mathcal{E})||_2 \leq C_1 \nu \sqrt{p d log(d)} \;\;\text{ and }\;\;||\mathcal{A}^*_\Omega(\mathcal{E})||_\infty\leq C_2 \nu \sqrt{p log(d)}$$
\end{lemma}
We now present an error bound of for a stationary point of Algorithm 1 applied to matrix completion.
\begin{theorem} \label{thm:mc}
Let $b= \mathcal{A}_\Omega X^*+\mathcal{E}$ for a rank $r$, $\mu-$ incoherent matrix $X^* \in \mathbb{R}^{d_1\times d_2}$, and the entries of $\mathcal{E}$ are mean zero i.i.d.\ random variables with variance $\nu^2$.  There exists universal constants $C_1$ and $C_2$ such that under the same assumptions as Lemma \ref{lem:recht}, if $\lambda> C_1 \nu \sqrt{p d \text{ log}(d)}$, then the iterates of Algorithm 1 linearly converge to a point $\bar{X}$ such that $||X^*-\bar{X}||_F^2 $ is less than: 
$$C_2 \frac{d_1d_2}{n}\big(\underbrace{ r\nu^2 d \text{ log}(d)}_{\begin{array}{c}
      \text{ optimal} \\
     \text{error rate}
\end{array}}+ \underbrace{ \lambda r \phi'(\sigma_{r}(X^*))}_{\text{bias term}}\big)$$
with convergence rate $\frac{1}{6}\frac{d_1d_2}{n}$. 
\end{theorem}

 The two terms of the error bound account for the optimal error rate and a bias term.  The optimal error rate is the error bound if we know the tangent space of $X^*$ a priori, that is, the difference between $X^*$ and the solution to the optimization problem $$\underset{X\in \mathcal{T}}{\text{min }}\;||\mathcal{A}_\Omega(X)-b||_F^2.$$  The oracle rate is further discussed in \cite{candes_plan_2010} and \cite{negahban_wainwright_2012}.
 
 \begin{proof}
 We can apply Theorem 1 with $s^k=s^*=0$, $\tau_L=1$, and $\kappa_L=\frac{1}{6}$ (by Lemma 3) to get the bound 
$$||\Delta^{k+1}||_F^2 \leq \frac{1}{6} ||\Delta^{k}||_F^2 + \big(\frac{d_1d_2}{n} ||\mathcal{P}_\mathcal{T}\mathcal{A}_\Omega^* \mathcal{E}||_F^2 + \lambda r \phi'(\sigma_r(X^*) -2\lambda)\big)$$
Initializing with $X^0=0$, we have the error at each iteration as follows:
$$||\Delta^k||^2\leq \big(\frac{1}{6}\big)^k||X^*||_F^2+\frac{6}{5}\big(1-\big(\frac{1}{6}\big)^k\big) \big(\frac{d_1d_2}{n} ||\mathcal{P}_\mathcal{T}\mathcal{A}_\Omega^* \mathcal{E}||_F^2 + \lambda r \phi'(\sigma_r(X^*) -2\lambda)\big)$$
Taking the limit as $k\rightarrow\infty$, and applying the bounds from Lemma \ref{lem:noise} gives the desired result.
 \end{proof}

Perhaps counter-intuitively, a choice of step size, $\tau$, that minimizes the loss function (i.e. the steepest descent step size) is not always the step size that leads to the fastest convergence rate.  To see this, we compare the error bound at iteration $k$ in both cases.  The steepest descent step size ($\tau=n$) would give 
\begin{align*}
    L^*-\tilde{L}^{k}&= L^* -L^{k-1}-  \mathcal{P}_\Omega(L^*-L^{k-1})=P_{\Omega^C}(\Delta_L)
\end{align*} 
where $\Omega^C$ is the set of indices not in $\Omega$.  The stepsize informed by the RIP ($\tau=\frac{d_1d_2}{n}$) gives
\begin{align*}
    L^*-\tilde{L}^{k}&= L^* -L^{k-1}-\frac{d_1d_2}{n} \mathcal{P}_\Omega(L^*-L^{k-1})=\Delta_L-\frac{d_1d_2}{n}P_\Omega(\Delta_L).
\end{align*} 

While the $\tau=1$ gives a significantly smaller error when simply comparing $L^*-\tilde{L}^k$, the error bound for $L^*-L^k$ comes from projecting $L^*-\tilde{L}^k$ onto $\mathcal{T}$.  Without further information, the best bound we can get when using $\tau=n$ would be 
$$||L^*-L^{k}||_F^2 \leq  1- \frac{(1-\kappa)n}{d_1d_2} ||L^*-L^{k-1}||_F^2.$$
This convergence rate approaches 1 asymptotically when we consider the information theoretic minimum number of measurements for large $d_1$ and $d_2$.  However, when $\tau=\frac{d_1d_2}{n}$, the convergence rate remains constant:
\begin{align*}
    ||L^*-L^{k}||_F^2 &\leq  ||\Delta_L-\frac{d_1d_2}{n}\mathcal{P}_\mathcal{T}\mathcal{P}_\Omega \mathcal{P}_\mathcal{T}(\Delta_L)||_F^2\leq \kappa ||\Delta_L||_F^2
\end{align*}
The first inequality uses the fact that $\Delta_L\in \mathcal{T}$ and the second uses the RIP.

\subsection{Robust PCA}
Next, we will use Theorem \ref{thm:main} to analyze the APGD algorithm applied to the problem of RPCA.  Specifically, we are interested in the special case of \eqref{eqn:nonconvex_formulation} where the $A_s=I_n$, and $\mathcal{A}_L=\mathcal{A}_{\Omega^{obs}}$.

In order for RPCA to be possible, we need the nonzero entries of $\mathcal{A}_{\Omega^{obs}}^*s^*$ to be sufficiently well-distributed throughout the rows and columns -- if the sparse corruptions affected the same row or column of $L^*$, then this would also be a low-rank perturbation and thus be impossible to separate from $L^*$ without further information.  So, we will assume that $\mathcal{A}_L^*s^*$ is $\alpha-$sparse, defined as follows.  

\begin{definition}
The matrix $S$ is $\alpha$-sparse for $0<\alpha<1$ if the proportion of nonzero entries in any row or column is less than $\alpha$. That is,
\begin{equation}
  ||S_{i:}||_0 \leq \alpha d_1 \:,   ||S_{:j}||_0 \leq \alpha d_2\; \forall i,j
\end{equation}
\end{definition}

In order to verify that the ROP property holds, we present the following Lemma:
\begin{lemma}\label{lem:ROP}
Let $\mathcal{T}$ be a rank $r$, $\mu-$incoherent tangent space, and let $\Omega$ be an $\alpha-$sparse subspace.  Then,
\begin{equation}
    ||\mathcal{P}_\mathcal{T}\mathcal{P}_\Omega||\leq 2\alpha \mu r, \;||\mathcal{P}_\Omega\mathcal{P}_\mathcal{T}||\leq 2\alpha \mu r
\end{equation}
\end{lemma}
\begin{proof}
By the triangle inequality, for a matrix $S\in \Omega$,
\begin{align*}
    ||\mathcal{P}_\mathcal{T}(S)||_F^2 \leq||UU^TS||^2_F+||SVV^T||_F^2 +||UU^TSVV^T||_F^2
\end{align*}
Because $U$ is an orthonormal matrix, 
\begin{align*}
    ||UU^TS||^2_F=&\text{trace}(UU^TS S^TUU^T)=\text{trace}(U^TUU^TS S^TU)\\=&\text{trace}(U^TS S^TU)=||U^TS||_F^2
\end{align*}
We now expand this norm and use the incoherence property to obtain the desired result:
\begin{align*}
    ||U^TS||_F^2=&\sum_{k=1}^r \sum_{i=1}^{d_1} \langle U_k, S_i \rangle^2
     \leq \sum_{k=1}^r \sum_{i=1}^{d_1} \big(  \sum_{j\in \Omega_i} U_{kj}^2\big)||S_i||^2\\
     =&\sum_{i=1}^{d_1} ||S_i||^2 \sum_{j\in \Omega_i} ||U_j||^2
     \leq \sum_{i=1}^{d_1} ||S_i||^2 \alpha \mu r \leq \alpha \mu r ||S||_F^2
\end{align*}
\end{proof} 
We can now give a bound for the stationary point of the APGD algorithm applied to RPCA.
\begin{theorem}
\label{thm:rpca}
Let $b= \mathcal{A}_{\Omega^{obs}}(L^*)+s^*+\mathcal{E}$ for a rank $r$, $\mu-$ incoherent matrix $L^* \in \mathbb{R}^{d_1\times d_2}$, a sparse vector $s^* \in \mathbb{R}^n$ with $||s^*||_\infty \leq 2 ||L^*||_\infty$, and a vector $\mathcal{E}\in \mathbb{R}^n$ whose entries are independent Gaussian variables with mean 0 and standard deviation $\sigma$.  Under the same assumptions on $\Omega^{obs}$ as Lemma \ref{lem:recht}, and assuming that  $\mathcal{A}_\Omega^* s^*$ is $\alpha-$ sparse and $\alpha \mu r \leq \frac{1}{64}$, if $\lambda_L\geq \frac{1}{6}+||\mathcal{A}_\Omega^* \mathcal{E}||_2$ and $\lambda_s\geq \frac{\mu r}{d_1d_2}+ ||\mathcal{E}||_\infty $, then the iterates of Algorithm 1 linearly converge to a point $\bar{L}, \bar{s}$ satisfying 
\begin{align*}
    ||\Delta_L||_F^2\leq C_1 \frac{d_1d_2}{n}r \nu^2 d \text{ log}(d), \;\;\:
    ||\Delta_S||_F^2\leq C_2\frac{d_1d_2}{n}r \nu^2 d \text{ log}(d)
\end{align*}
with convergence rate $\frac{1}{6}$.

\end{theorem}
\begin{proof}
In order to apply Lemma 4, we first show the following properties about the parameters $\lambda_L$ and $\lambda_s$.
\begin{align*}
    \lambda_L &\geq ||\mathcal{A}_L^* \mathcal{E}||_2+ ||\mathcal{A}_L^* s^*||_2\\
    &\geq ||\mathcal{A}_L^* \mathcal{E}||_2+ ||\mathcal{A}_L^* \delta_s^k||_2 \; \forall k\\
\end{align*}
where the first inequality comes from the assumption on $\lambda_L$ and the second from the fact that the two norm of $s^k$ is decreasing.
\begin{align*}
    \lambda_s \geq& || \mathcal{E}||_\infty+ \frac{\mu r}{d_1 d_2} \\
         \geq& || \mathcal{E}||_\infty+ ||L^*||_\infty \\
         \geq& || \mathcal{E}||_\infty+ ||\Delta_L^k||_\infty \;\forall k
\end{align*}
where the first inequality is our assumption on $\lambda_s$,  the second follows from incoherence of $L^*$, and the last from the fact that $$||\Delta_L^k||_\infty \leq ||\Delta_L^0||_\infty =||L^*||_\infty.$$

In order to show the bias term is 0, we use our assumption that $\lambda_L \leq \frac{1}{4}\sigma_r(L^*)$ and the fact that the MCP regularizer can satisfy $\phi_L'(3 \lambda_L)=0$ while still maintaining $\frac{1}{2\lambda_L}$-weak convexity.  Likewise, when we assume $\lambda_s\leq \frac{1}{4} s_{min}$, we can choose $\phi_s$ such that $\phi'_s(s_{min}-\lambda_s) \leq \phi'_s(3\lambda_s) =0$ while still maintaining $\frac{1}{2\lambda_s}$ weak convexity. 

We can now apply the first part of Theorem \ref{thm:main} to obtain the bound:
\begin{align*}
    ||\Delta^{k+1}_L||_F^2&\leq  \frac{1}{6}||\Delta^k_L||_F^2+\frac{1}{64}||\Delta^k_s||_F^2 +\frac{d_1d_2}{n}||\mathcal{P}_\mathcal{T} \mathcal{A}_L^* \mathcal{E} ||_F^2\\
   ||\Delta_S^{k+1}||_F^2&\leq  \frac{1}{64}||\Delta_L^{k+1}||_F^2+||\mathcal{P}_\Omega \mathcal{E} ||_F^2\\
    &\leq \frac{1}{384}||\Delta_L^k||_F^2+\frac{1}{4096}||\Delta_S^k||_F^2+\frac{d_1d_2}{64n}||\mathcal{P}_\mathcal{T} \mathcal{A}_L^* \mathcal{E} ||_F^2+||\mathcal{P}_\Omega \mathcal{E} ||_F^2
\end{align*}
The limit point of this sequence gives:
\begin{align*}
    ||\Delta_L||_F^2&\leq  \frac{6}{5} \frac{d_1d_2}{n}||\mathcal{P}_\mathcal{T} \mathcal{A}_L^* \mathcal{E} ||_F^2\\
    ||\Delta_S||_F^2&\leq \frac{d_1d_2}{32n}||\mathcal{P}_\mathcal{T} \mathcal{A}_L^* \mathcal{E} ||_F^2+\frac{6}{5}||\mathcal{P}_\Omega \mathcal{E} ||_F^2
\end{align*}
From here, we can apply the bounds from Lemma \ref{lem:noise} to get the desired results.
\end{proof}

The bound presented in Theorem \ref{thm:rpca} are a major improvement upon previous bounds, such as the ones presented in Agarwal, Negahban, and Wainwright \cite{agarwal_negahban_2012}.  For the estimator \eqref{eqn:rpca_reg}, they require that $$\lambda_s \geq 4(||L^*||_\infty+||E||_\infty). $$  However, because the sparse matrix $S$ obtained from their method will be zero whenever the true value is less than $\lambda_s$, their result has an implicit assumption that the nonzero values of $S^*$ are four times larger than any value in $L^*+E$.  If this were true, then identifying entries of the measured matrix $M$ could be accomplished by identifying the entries with the largest absolute value. Our result imposed the much less strict assumption that the nonzero entries of $S^*$ are larger than $||E||_\infty$ plus a small constant (an assumption that is necessary for $S^*$ and $E$ to be separable).

\section{Numerical Results}
We implemented Algorithm 1 in Matlab R2020a, for which the code is available at GitHub. All results in this section are obtained with the Matlab version in order to accurately compare to other algorithms which are only available in Matlab, and are run on a Windows 10 desktop with an AMD Phenom 3.40 GHz processor and 8 Gb of RAM.

\subsection{Matrix Completion}

\begin{table}
\centering
\caption{Comparison of four different matrix completion algorithms on randomly generated low rank matrices and common recommendation data sets.  The algorithm LMaFit reconstructs a matrix of a given rank $k$.  The table shows the results when the algorithm is given the exact rank($k=r$) and an incorrect rank ($k=2r$).  Time is given in seconds }\label{table:mc} 
\begin{tabular}{|l|l|l|l|l|l|l|l|l|}
\hline
$d_1$         & \multicolumn{2}{c|}{1000} & \multicolumn{2}{c|}{1000} & \multicolumn{2}{c|}{5000} & \multicolumn{2}{c|}{5000} \\ \hline
$d_2$         & \multicolumn{2}{c|}{500}  & \multicolumn{2}{c|}{500}  & \multicolumn{2}{c|}{1000} & \multicolumn{2}{c|}{1000} \\ \hline
$r$           & \multicolumn{2}{c|}{5}    & \multicolumn{2}{c|}{5}    & \multicolumn{2}{c|}{10}   & \multicolumn{2}{c|}{10}   \\ \hline
$n/d_1 d_2$   & \multicolumn{2}{c|}{0.3}  & \multicolumn{2}{c|}{0.1}  & \multicolumn{2}{c|}{0.2}  & \multicolumn{2}{c|}{0.05} \\ \hline
std($E$)      & \multicolumn{2}{c|}{0.1}  & \multicolumn{2}{c|}{0.02} & \multicolumn{2}{c|}{0.1}  & \multicolumn{2}{c|}{0.02} \\ \hline
              & RFNE          & T         & RFNE          & T         & RFNE          & T         & RFNE          & T         \\ \hline
APGD          & 3.28e-4       & 0.7       & 2.90e-4       & 1.1       & 1.69e-4       & 10        & 1.96e-4       & 29        \\ \hline
FaNCL         & 3.28e-4       & 1.8       & 2.90e-4       & 4.3       & 1.69e-4       & 31        & 2.49E-4       & 57        \\ \hline
IALM          & 3.28e-4       & 2.6       & 2.92e-4       & 2.8       & 1.72e-4       & 32        & 1.99e-4       & 27        \\ \hline
LMaFit (2$r$) & 4.68e2        & 21        & 1.20e3        & 6.2       & 4.99e2        & 197       & 1.60e3        & 40        \\ \hline
LMaFit($r$)   & 3.28e-4       & 0.4       & 2.90e-4       & 0.3       & 1.69e-4       & 3.9       & 1.96e-4       & 3.9      
\\
\hline
\end{tabular}
\end{table}

\begin{table}[] \caption{Comparison of four different matrix completion algorithms on two common recommender system benchmarking datasets. }
\centering
\begin{tabular}{l|c|l|l|l|}
                                                                   & \multicolumn{2}{c|}{Jester Dataset}   & \multicolumn{2}{c|}{ML 1M}                  \\ \hline
\multicolumn{1}{|l|}{Number of Users ($d_1$)}                      & \multicolumn{2}{c|}{24983}            & \multicolumn{2}{c|}{6040}                   \\ \hline
\multicolumn{1}{|l|}{Number of Items ($d_2$)}                      & \multicolumn{2}{c|}{100}              & \multicolumn{2}{c|}{3952}                   \\ \hline
\multicolumn{1}{|l|}{Percentage of Entries Observed ($n/d_1 d_2$)} & \multicolumn{2}{c|}{0.58}             & \multicolumn{2}{c|}{0.034}                  \\ \hline
\multicolumn{1}{|l|}{}                                             & \multicolumn{1}{l|}{NMAE}  & Time (s) & NMAE                             & Time (s) \\ \hline
\multicolumn{1}{|l|}{APGD}                                         & \multicolumn{1}{l|}{0.159} & 21       & 0.172 & 172      \\ \hline
\multicolumn{1}{|l|}{FaNCL}                                        & \multicolumn{1}{l|}{0.183} & 42       & 0.200                            & 42       \\ \hline
\multicolumn{1}{|l|}{IALM}                                         & \multicolumn{1}{l|}{0.163} & 77       & 0.183                            & 216      \\ \hline
\multicolumn{1}{|l|}{LMaFit}                                  & 0.168                      & 7.4      & 9.174                            & 44 \\      \hline
\end{tabular}
\end{table}

We compare our method for matrix completion to another method utilizing nonconvex regularizer from \cite{yao_kwok_2019} (FaNCL), along with a method to minimize the nuclear norm IALM, from \cite{ialm}, and a rank constrained method, LMaFit \cite{lmafitMC}.  The results are shown in Table \ref{table:mc}.

We start by comparing the performance of the methods on randomly generated low rank matrices of varying size, rank, percentage of observed entries, and standard deviation of the noise in the measurements (shown relative to the mean absolute value of the low rank matrix).  In each of the cases, our method performs exactly as well as FaNCL and LMaFit when the correct rank is given.  IALM performs equally well in the first case, and slightly worse than the remaining three cases due to the fact that the nuclear norm biases the result towards zero.  

Next, we show the results on common data sets for recommendation systems, the Jester data set \cite{Jester} and MovieLens 1M \cite{movielens}.  For each of these two data sets, we partition the observations into five folds, fit a low rank model to four of the folds and calculate the accuracy on the remaining fold. We repeat this for each of the five folds and present the average normalized mean absolute error.  In both cases, our method outperforms the other three algorithms we compare to.

\subsection{Robust PCA}

We compare our method to several other prominent RPCA methods, including LMaFit \cite{lmafitRPCA}, AltProj \cite{netrapalli2014non}, RPCA-GD \cite{yi_park_2016}, and IALM \cite{ialm}.  We compared with many other methods included in the LRSLibrary \cite{lrslibrary2015}, however we only include results from the aforementioned algorithms as they gave the most accurate results for matrices with a significant amount with noise, a test case we emphasise in this section.  

It is worth noting that, while our algorithm does not require an estimate of the rank of $L^*$ or the sparsity of $S^*$ a priori, RPCA-GD requires both, and LMaFit and AltProj require an estimate of the rank.  However, we found that LMaFit and AltProj still perform very well when this estimate is unreliable, as LMaFit includes a rank-estimation scheme and AltProj starts by performing a rank 1 projection, and increases the rank up until the estimate given.  We provide all methods with an upper bound on the rank equal to twice the rank of $L^*$ and an upper bound on the number of corrupted entries equal to twice that of $S^*$.  

\begin{table}
\caption{Accuracy and run-time for RPCA on four videos of fish.}\label{table:fish}
\centering
\begin{tabular}{l|l|l|l|l|l|l|l|l|l|l|}
\cline{2-11}
                              & \multicolumn{2}{c|}{Marine Snow} & \multicolumn{2}{c|}{Aquaculture} & \multicolumn{2}{c|}{Caustics} & \multicolumn{2}{c|}{Two Fish} & \multicolumn{2}{c|}{Fish Swarm} \\ \cline{2-11} 
                              & Acc          & Time     & Acc             & Time        & Acc         & Time  & Acc         & Time   & Acc          & Time    \\ \hline
\multicolumn{1}{|l|}{APGD}    & 0.92              & 87           & \textbf{0.92}        & 592             & \textbf{0.88}    & 459        & 0.90              & 467        & \textbf{0.75}     & 473         \\ \hline
\multicolumn{1}{|l|}{LMaFit}  & 0.59              & 258          & 0.65                 & 316             & 0.7              & 273        & 0.62             & 291        & 0.55              & 310         \\ \hline
\multicolumn{1}{|l|}{AltProj} & 0.84              & 143          & 0.62                 & 109             & 0.86             & 83         & 0.76             & 109        & 0.62              & 130         \\ \hline
\multicolumn{1}{|l|}{RPCA-GD} & \textbf{0.92}     & 94           & 0.72                 & 77              & 0.87             & 76         & \textbf{0.96}    & 96         & 0.66              & 93          \\ \hline
\multicolumn{1}{|l|}{IALM}    & 0.57              & 379          & 0.56                 & 286             & 0.56             & 288        & 0.56             & 342        & 0.56              & 305         \\ \hline
\end{tabular}
\end{table}

\begin{figure}
    \centering
    \includegraphics[width=\columnwidth]{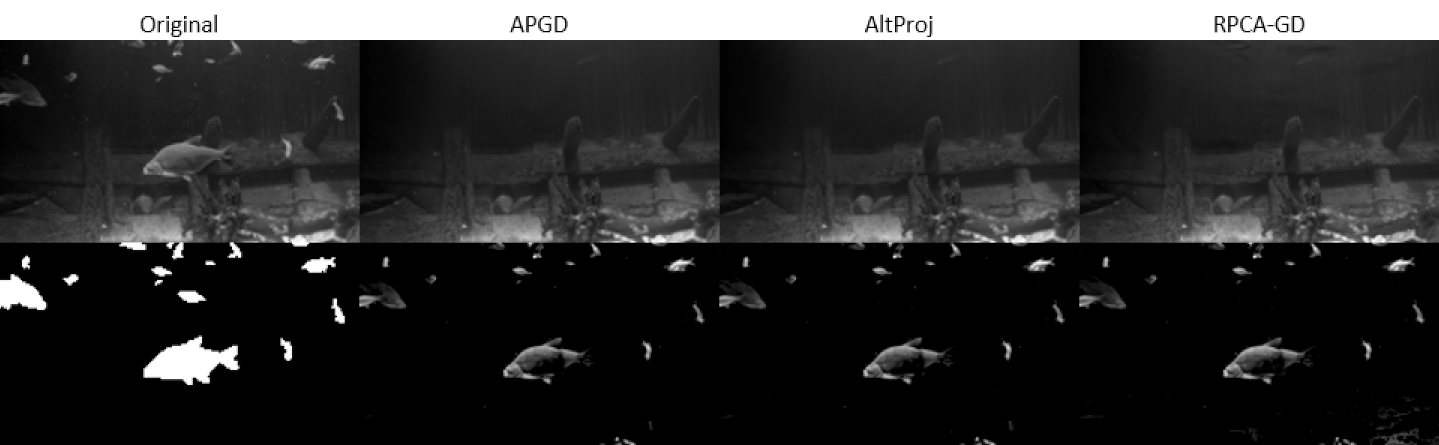}
    \caption{Frames from the Marine Snow video.  The first row of the first column shows the original frame, with the image segmentation for that frame below it.  The remaining three columns show the background and foreground obtained by three different methods. }
    \label{fig:marine_snow}
\end{figure}

\begin{figure}[ht]
    \centering
    \includegraphics[width=\columnwidth]{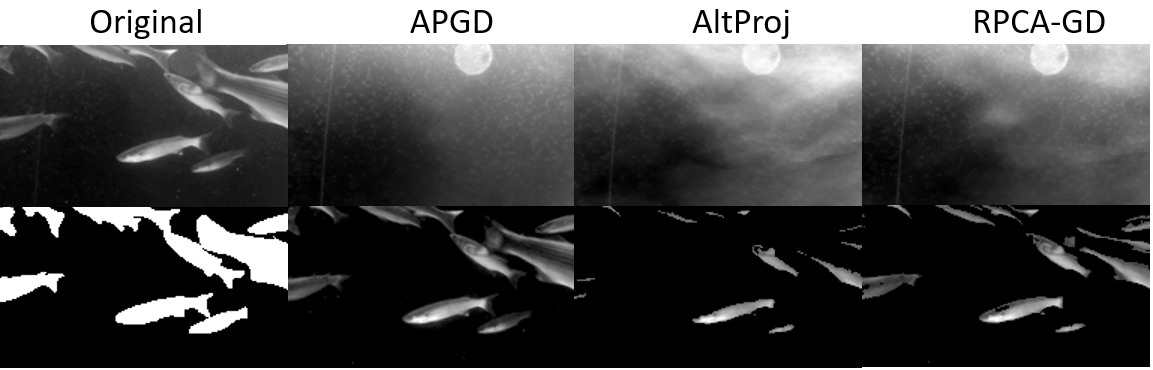}
    \caption{Frames from the Small Aquaculture video.}
    \label{fig:small_aquaculture}
\end{figure}
\begin{figure}[ht]
    \centering
    \includegraphics[width=\columnwidth]{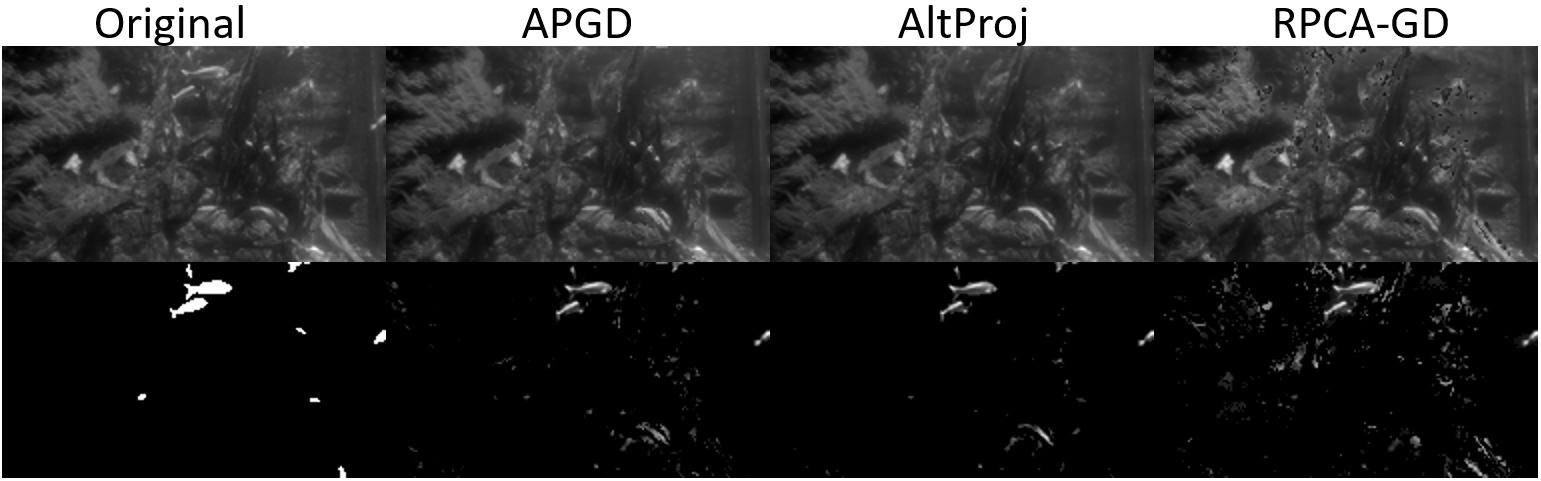}
    \caption{Frames from the Small Aquaculture video.}
    \label{fig:caustics}
\end{figure}\begin{figure}[ht]
    \centering
    \includegraphics[width=\columnwidth]{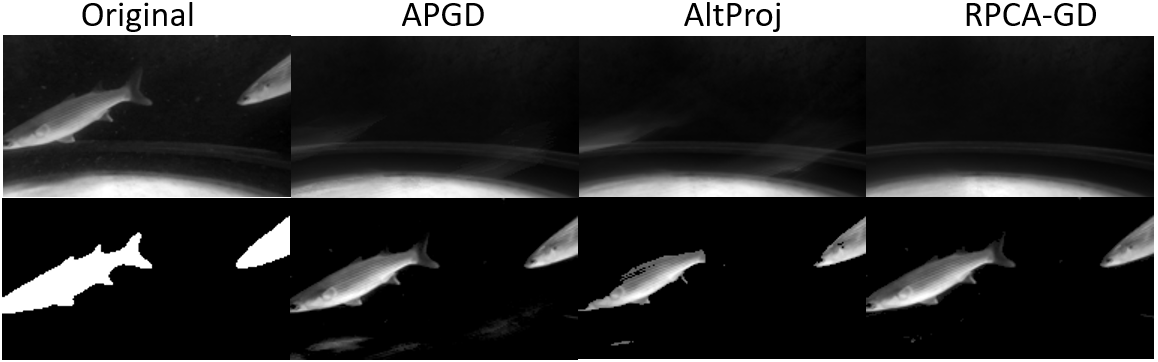}
    \caption{Frames from the Small Aquaculture video.}
    \label{fig:two_fish}
\end{figure}\begin{figure}[ht]
    \centering
    \includegraphics[width=\columnwidth]{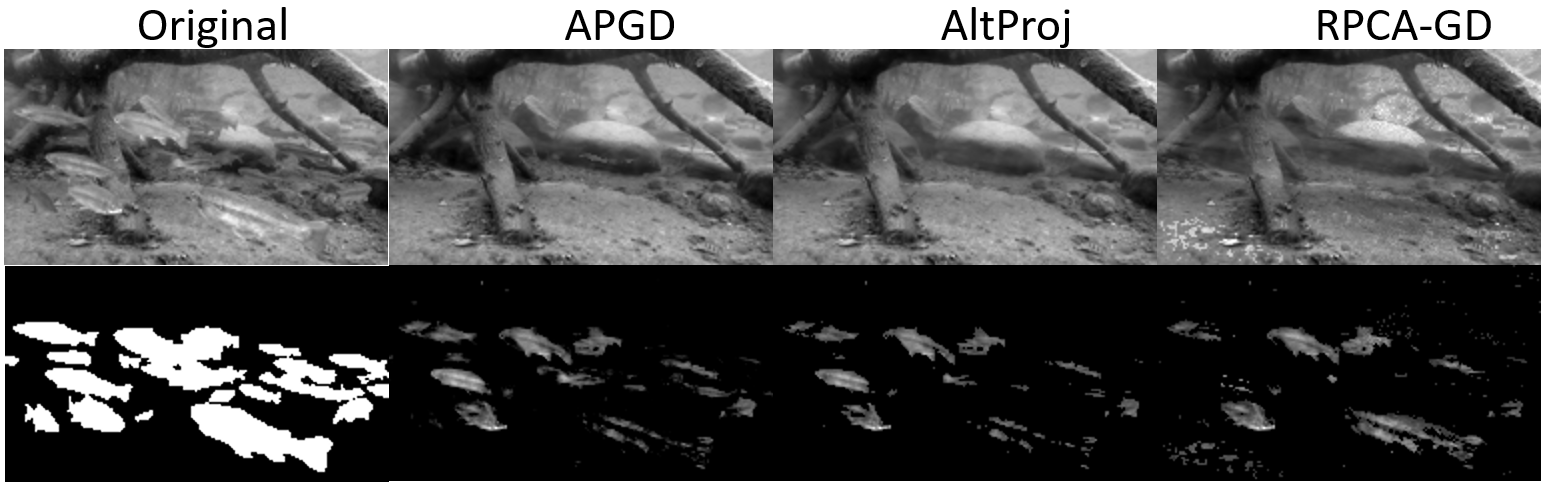}
    \caption{Frames from the Small Aquaculture video.}
    \label{fig:fish_swarm}
\end{figure}
The most commonly used test case in the literature for RPCA is the task of separating the background and foreground of a video.  In this scenario, each frame of the video is represented as a column vector in $M$.  If the background is static (or, at least, in some way repetitive), then we can expect the background of each frame to be represented as a low rank matrix, and the foreground as a sparse matrix.  See \cite{candes_li_2011} for further details.

We evaluated our algorithm on the Underwater Change Detection dataset \cite{fish}.  Out of the 1100 frames in each video, the ground truth image segmentation is included for last 100 frames.  This allows use to present an objective and accurate metric of how well each method is able to identify the foreground of the image. 

In Table \ref{table:fish}, we present the runtime and the accuracy of determining which pixels contain a fish for our method compared to the four previously mentioned approaches.  To calculate the accuracy, we average of the true positive ratio and true negative ratio.  In three of the five videos, our method achieves the highest accuracy, whereas in the other two the RPCA-GD algorithm preforms slightly better.  The recovered background and foreground for our method, RPCA-GD and AltProj are shown in Figures \ref{fig:marine_snow}, \ref{fig:small_aquaculture}, \ref{fig:caustics}, \ref{fig:two_fish}, \ref{fig:fish_swarm}, along with the original frame and segmented image.
\section{Conclusions}
We have shown a novel convergence analysis of the alternating proximal gradient descent algorithm applied to the problems of matrix completion and RPCA with nonconvex regularizers, and bound the difference from the ground truth low rank matrix and sparse vector.  Future work on the topic could include extending our analysis to data that lies on more complicated, nonlinear manifolds.

\bibliographystyle{siamplain}

\end{document}